\documentclass{article}

\usepackage{arxiv}

\usepackage[utf8]{inputenc} 
\usepackage[T1]{fontenc}    
\usepackage{hyperref}       
\usepackage{url}            
\usepackage{booktabs}       
\usepackage{amsfonts}       
\usepackage{nicefrac}       
\usepackage{microtype}      
\usepackage{amsthm}
\usepackage{amsmath}
\usepackage{amssymb,dsfont,bbold,enumitem,mathtools}
\usepackage[dvipsnames]{xcolor} 

\usepackage{caption}
\usepackage{wrapfig}

\newcommand{\xhdr}[1]{\vspace{0.0mm}\noindent{\textbf{#1.}}\hspace{0.5mm}}

\usepackage{enumitem}

\usepackage{algorithmic}
\usepackage[ruled,linesnumbered]{algorithm2e}

\newtheorem{theorem}{Theorem}[section]
\newtheorem{lemma}[theorem]{Lemma}

\newcommand{\RR}{\mathbb{R}}

\def\A{\mathcal A}
\def\P{\mathcal P}

\def\tilde{\widetilde}
\def\ep{\varepsilon}
\def\algname{NHOLS}
\def\S{\mathcal S}

\def\L{\mathit \Delta}

\usepackage{hyperref}
\definecolor{mylinkcolor}{RGB}{0,0,140}
\hypersetup{colorlinks,allcolors=mylinkcolor,citecolor=mylinkcolor}

\begin{document}

\title{Nonlinear Higher-Order Label Spreading}

\author{%
  Francesco Tudisco
   \\
  School of Mathematics\\
  Gran Sasso Science Institute\\
  67100, L'Aquila Italy \\
  \texttt{francesco.tudisco@gssi.it} 
   \And
   Austin R. Benson \\
   Department of Computer Science\\
   Cornell University\\
   Ithaca, NY 14853 \\
   \texttt{arb@cs.cornell.edu} 
   \And
   Konstantin Prokopchik \\
   School of Computer Science \\
   Gran Sasso Science Institute \\
   67100, L'Aquila Italy\\
   \texttt{konstantin.prokopchik@gssi.it} 
}

\maketitle

\begin{abstract}
Label spreading is a general technique for semi-supervised learning with point
cloud or network data, which can be interpreted as a diffusion of labels on a
graph. While there are many variants of label spreading, nearly all of them are
linear models, where the incoming information to a node is a weighted sum of
information from neighboring nodes. Here, we add nonlinearity to label spreading
through nonlinear functions of higher-order structure in the graph, namely
triangles in the graph. For a broad class of nonlinear functions, we prove
convergence of our nonlinear higher-order label spreading algorithm to the
global solution of a constrained semi-supervised loss function. We demonstrate
the efficiency and efficacy of our approach on a variety of point cloud and
network datasets, where the nonlinear higher-order model compares favorably to
classical label spreading, as well as hypergraph models and graph neural
networks.
\end{abstract}

\section{Introduction}

Label Spreading (LS) is a general algorithmic technique for Semi-Supervised
Learning (SSL), where one infers unknown labels from known labels by iteratively
diffusing or ``spreading'' the known labels over a similarity graph where nodes
correspond to data points and are connected by edges if they are
similar~\cite{zhu2009introduction}. With generic point cloud data, edges are
typically $k$-nearest-neighbors or
$\epsilon$-neighbors~\cite{joachims2003transductive,zhou2004learning,zhu2003semi},
but LS can also be used directly on relational data coming from, e.g., social
networks~\cite{chin2019decoupled}, web graphs~\cite{kyng2015algorithms}, or
co-purchases~\cite{gleich2015using}.  A canonical or ``standard'' implementation
of label spreading is the iterative local and global consistency
approach~\cite{zhou2004learning}. In this method, nodes iteratively spread their
current label to neighbors (encouraging local consistency), and originally
labeled points try to remain close to their initial labels (encouraging global
consistency). This procedure corresponds to a linear gradient-flow that
minimizes a regularized loss function in the limit.

The linearity of the diffusion makes standard LS simple to
implement. From the perspective of an unlabeled data point, the corresponding
node in the graph iteratively updates its label based on a fixed \emph{linear}
combination (the mean) of the current labels of its neighbors in the graph.  The
linearity also makes it easy to analyze the limiting behavior of this iterative
process, which coincides with the solution of a Laplacian linear system. At the
same time, nonlinear methods provide stronger modeling capabilities, making them
a hallmark of deep learning~\cite{Goodfellow-et-al-2016} and kernel
methods~\cite{hofmann2008kernel}.

Here, we incorporate nonlinearity into label spreading via nonlinear functions
on so-called ``higher-order'' relationships between the data points, i.e.,
information about groups of nodes instead of just the similarities encoded by
edges in a graph. More specifically, we use a similarity \emph{hypergraph} that
encodes higher-order relationships. This hypergraph can come directly from data
or be derived from a similarity graph (we often use the hypergraph induced by
the 3-cliques in a similarity graph). From this, we devise a new spreading
process, where a given node $u$ updates its label based on the labels of the
other nodes in the hyperedges that contain $u$. Importantly, we allow the
spreading of information at a hyperedge to be any of a broad class of nonlinear
``mixing functions.''

We call our approach Nonlinear Higher-Order Label Spreading (\algname) since it
uses both nonlinear and higher-order information. Even though the nonlinearity
of our spreading process makes analysis more challenging, we show that
{\algname} enjoys several nice properties similar to standard LS.
First, {\algname} minimizes a regularized loss function which
combines a global consistency term with a higher-order local consistency term.
Second, for a broad class of nonlinear mixing functions, {\algname}
globally converges to a unique global minimizer of this loss function.
Furthermore, in terms of implementation, {\algname} shares the same simplicity
and efficiency as standard LS. Each iteration only requires a single pass over the
input data, making it highly scalable.

We evaluate {\algname} on a number of synthetic and real-world datasets,
comparing against standard LS, hypergraph semi-supervised learning methods, and
graph neural networks. We find that incorporating nonlinearities of
higher-order information into label spreading almost always achieves
the best performance, while also being nearly as fast as standard LS.

\subsection{Related work}

\xhdr{Higher-order information for graph data} A key idea in many recent
graph-based learning methods is that incorporating \emph{higher-order}
interactions involving multiple nodes can make large changes in
performance. This has yielded improvements in numerous settings, including
unsupervised clustering~\cite{benson2016higher,li2017inhomogeneous,tsourakakis2017scalable},
localized clustering~\cite{li2019optimizing,yin2017local},
representation learning~\cite{rossi2018higher},
link prediction~\cite{arrigo2020framework,benson2018simplicial,rossi2019higher},
graph classification~\cite{ahmed2017graphlet},
ranking~\cite{arrigo2020framework,arrigo2019multi,benson2019three,chitra2019random}, and
data visualization~\cite{nassar2020using}.
A higher-order version of label spreading has also recently been developed for
relational data~\cite{eswaran2020higher}, and this correspond to the special
case of linear mixing functions within our framework.

\xhdr{Hypergraph learning}
There are many machine learning methods for hypergraph data, and a standard
approach  is to first reduce the hypergraph to a graph upon which a
graph-based method can be
employed~\cite{agarwal2006higher,feng2019hypergraph,li2017inhomogeneous,satchidanand2015extended,zhang2018beyond,zhou2007hypergraph}.
These techniques are ``clique expansions'', as they place a (possibly weighted)
clique in the graph for every hyperedge. Using a linear mixing function in our
framework is a clique expansion technique, which we cover in
Section~\ref{sec:nhols}. Our analysis is focused on nonlinear mixing functions,
which is not a clique expansion. Thus, our framework is conceptually closer to
hypergraph methods that avoid clique expansions, such as those based on
nonlinear Hypergraph Laplacian
operators~\cite{chan2018spectral,li2018submodular,yadati2019hypergcn} or
generalized splitting functions~\cite{veldt2020hypergraph,veldt2020localized}.

\xhdr{Nonlinear semi-supervised learning}
There are variants of label spreading techniques that use nonlinearities, such
as $p$-Laplacians~\cite{alamgir2011phase,bridle2013p,ibrahim2019nonlinear}
and their limits~\cite{kyng2015algorithms}. Our theoretical framework provides
new convergence guarantees for some of these approaches.

\xhdr{Tensor methods}
Tensors can also represent higher-order data and are broadly used in
machine learning~\cite{anandkumar2014tensor,jaffe2018learning,nguyen2016efficient,sidiropoulos2017tensor,wu2016general}.
We analyze the iterations of \algname{} as a type of tensor contraction;
from this, we extend recent nonlinear Perron-Frobenius theory~\cite{gautier2019contractivity,gautier2019perron} to establish 
convergence results.

\section{Background on Standard Label Spreading}\label{sec:LS}
We first review a ``standard'' LS technique that is essentially the same
as the one of Zhou et al.~\cite{zhou2004learning} so that we can later draw parallels with our new
nonlinear higher-order method. Let $G=(V,E,\omega)$ be a weighted undirected graph
with nodes $V=\{1,\dots, n\}$, edge set $E \subseteq V\times V$ and edge-weight
function $\omega(ij)>0$.
As mentioned in the introduction, $G$ typically represents either a similarity graph for a point cloud
or a bona fide relational network.
Let $A$ be the adjacency matrix of $G$, i.e., by
$A_{ij}=\omega(ij)$ if $ij\in E$ and $A_{ij}=0$ otherwise.
Furthermore, let $D_G = \mathrm{Diag}(d_1, \dots, d_n)$ be the diagonal degree matrix of $G$,
where $d_i = \sum_j A_{ij}$. Throughout this paper, we will assume that $G$ has no isolated nodes, 
that is $D_G$ has no zero diagonal entries.
Finally, let $S = D_G^{-1/2}AD_G^{-1/2}$ be the normalized adjacency matrix.

Our goal is to provide a label in $\{1,\dots,L\}$ to each node, and we know the label
of (usually a small) subset of the nodes. The initial labels are represented by membership vectors
in an $n \times c$ matrix $Y$, where $Y_{i,\ell} = 1$ if node $i$ has initial label $\ell$
and $Y_{i,\ell} = 0$ otherwise.
Given an initial guess $F^{(0)}\in \RR^{n \times c}$, the label spreading algorithm iteratively computes
\begin{equation}\label{eq:LS}
 F^{(r+1)} = \beta SF^{(r)} + \gamma Y\, \quad r=0,1,2,\dots,
\end{equation}
with $\beta, \gamma \geq 0$ and $\beta+\gamma = 1$.
The iterates converge to the solution of the linear system $(I- \beta S)F^* = \gamma Y$,
but in practice, a few iterations of \eqref{eq:LS} with the initial point $F^{(0)} = Y$ suffices~\cite{fujiwara2014efficient}.
This yields an approximate solution $\tilde{F}^*$. 
The prediction on an unlabeled node $j$ is then $\arg\max_{\ell} \tilde{F}^*_{j,\ell}$
Also, in practice, one can perform this iteration column-wise with one initial vector $y$ per label class. 

The label spreading procedure \eqref{eq:LS} can be interpreted as
gradient descent applied to a quadratic regularized loss function and as a
discrete dynamical system that spreads the initial value condition $F^{(0)} = Y$
through the graph via a linear gradient flow. We briefly review these two
analogous formulations. Let $\psi$ be the quadratic energy loss function
that is separable on the columns of $F$:
\begin{equation}\label{eq:LS_loss}
 \psi(F) = \sum_{\ell=1}^{c} \psi_{\ell}(F_{:,\ell}) = \sum_{\ell=1}^{c} \frac 1 2 \big\{\|F_{:,\ell} -Y_{:,\ell}\|_2^2 + \lambda \,  F_{:,\ell}^\top \L F_{:,\ell} \big\},
\end{equation}
where $\L = I - D_G^{-1/2}AD_G^{-1/2}=I-S$ is the normalized Laplacian.
Consider the the dynamical system 
\begin{equation}\label{eq:ODE}
\dot f(t) = -\nabla \psi_{\ell} (f(t))
\end{equation}
for any $\ell \in \{1,\ldots,L\}$.
Since $\psi_{\ell}$ is convex, $\lim_{t\to\infty}f(t)=f^*$ such that
$\psi_\ell(f^*) = \min \psi_\ell(f)$. Label spreading coincides with gradient descent
applied to \eqref{eq:LS_loss} or, equivalently, with   explicit Euler integration
applied to \eqref{eq:ODE}, for a particular value of the step length $h$.
In fact,
\begin{equation}
f - h \nabla \psi_{\ell} (f) = f - h(f - Y_{:,\ell} + \lambda \L f) =(1- h - h\lambda) f + h\lambda Sf + h Y_{:,\ell},
\end{equation}
which, for $(1 - h) / h = \lambda$, coincides with one iteration of \eqref{eq:LS} applied to the $\ell$-th column of $F$. 
Moreover, as $F^{(r)} \geq 0$ for all $r$,
this gradient flow interpretation shows that the global minimizer of \eqref{eq:LS_loss} is  nonnegative, i.e.,
$\min \psi(F) = \min \{\psi(F) : F\geq 0\}$.
In the next section, we use similar techniques to derive our nonlinear higher-order label spreading method.

\section{Nonlinear Higher-order Label Spreading}\label{sec:nhols}
Now we develop our nonlinear higher-order label spreading (NHOLS) technique.
We will assume that we have a $3$-regular hypergraph $H = (V, \mathcal E, \tau)$ capturing higher-order information on the same set of nodes as a
weighted graph $G = (V, E, \omega)$, where $\mathcal E \subseteq V\times V \times V$ and $\tau$ is a hyperedge weight function $\tau(ijk)>0$.
In our experiments, we will usually derive $H$ from $G$ by considering the hyperedges of $H$ to be the set
of triangles (i.e., 3-cliques) of $G$. However, in principle we could use any hypergraph. We also do not need
the associated graph $G$, but we keep it for greater generality and find it useful in practice.
Finally, we develop our methodology for $3$-regular hypergraphs for simplicity and notational sanity, but our ideas generalize to arbitrary hypergraphs (Theorem~\ref{thm:general-convergence}).

\subsection{Nonlinear Second-order Label Spreading with Mixing Functions}
We represent $H$ via the associated third-order adjacency tensor $\A$, defined by $\A_{ijk} = \tau(ijk)$ if $ijk\in \mathcal E$ and $\A_{ijk}=0$ otherwise. 
Analogous to the graph case, let $D_H=\mathrm{Diag}(\delta_1,\dots, \delta_n)$ be the diagonal matrix of the hypergraph node degrees, where
$\delta_i = \sum_{j,k :\,  ijk\in\mathcal E}\tau(ijk) = \sum_{jk}\A_{ijk}$.
Again, we assume that $H$ has no isolated nodes so that $\delta_i>0$ for all $i\in V$.

As noted in the introduction, we will make use of nonlinear \emph{mixing functions},
which we denote by $\sigma\colon\RR^2\to \RR$.
For a tensor $T=T_{ijk}$, we define the tensor map $T\sigma \colon \RR^n\to\RR^n$ entrywise:
\begin{equation}
T\sigma(f)_{i} = \sum_{jk}T_{ijk}\, \sigma(f_j,f_k). \label{eq:contraction}
\end{equation}
Hence, in analogy with the matrix case, we denote by $\S:\RR^n\to\RR^n$ the map 
$$
\S(f) = D_H^{-1/2} \, \A\sigma(D_H^{-1/2}f)
$$
This type of tensor contraction will serve as the basis for our NHOLS technique.

We need one additional piece of notation that is special to the higher-order case,
which is a type of energy function that will be used to normalize iterates in order to guarantee convergence.
Let $B$ be the matrix with entries $B_{ij} = \sum_{k}\A_{kij}$. Define $\varphi \colon \RR^n \to \RR$ by
\begin{equation}\label{eq:varphi}
\varphi(f)  = \frac 1 2 \sqrt{\,\sum_{ij}B_{ij}\,  \sigma(\, f_i/\sqrt{\delta_i}\, ,\, f_j/\sqrt{\delta_j}\, )^2 }.
\end{equation}

Finally, we arrive at our nonlinear higher-order label spreading ({\algname}) method.
Given an initial vector $f^{(0)}\in \RR^n$, we define the NHOLS iterates by
\begin{equation}\label{eq:NLS}
g^{(r)} = \alpha \S(f^{(r)}) + \beta Sf^{(r)} + \gamma y,\quad
f^{(r+1)} = g^{(r)} / \varphi(g^{(r)}),\quad
r=0,1,2,\dots,
\end{equation}
where $\alpha,\beta,\gamma \geq 0$, $\alpha+\beta+\gamma = 1$, and
$y$ is an initial label membership vector.
Provided that $\sigma$ is positive (i.e., $\sigma(a,b) > 0$ for any $a, b > 0$)
and the initial vector $f^{(0)}$ is nonnegative, then all iterates
are nonnegative. This assumption on the mixing function will be crucial
to prove the convergence of the iterates.
We perform this iteration with one initial vector per label class,
analogous to standard Label Spreading.
Algorithm~\ref{alg:N2LS} gives the overall procedure.

\begin{algorithm}[t]
  \DontPrintSemicolon
	\caption{NHOLS: Nonlinear Higher-Order Label Spreading}\label{alg:N2LS}
        \KwIn{Tensor $\A$; matrix $A$; mixing function $\sigma:\RR^2\to\RR$;
          label matrix $Y \in \{0,1\}^{n \times c}$;
          scalars $\alpha,\beta,\gamma \geq 0$ with $\alpha+\beta+\gamma=1$;
          smoothing parameter $0 < \ep < 1$; stopping tolerance $\textnormal{tol}$} 
	\KwOut{Predicted labels $\hat{Y} \in \{1,\dots, c\}^{n}$}
        $\tilde{F} \in \RR^{n \times c}$ \quad \texttt{\# Store approximate solutions} \;
        \For{$\ell=1,\dots,L$}{
          $y_{\ep} \gets (1-\ep)Y_{:,\ell} + \ep \mathbf{1}$\;
          $f^{(0)} \gets y_{\ep}$ \quad \texttt{\# Initialize with label smoothing}\;
    	 \Repeat{$\|f^{(r+1)}-f^{(r)}\|/\|f^{(r+1)}\| < \textnormal{tol}$}{
    	        $g \gets  \alpha \S(f^{(r)}) + \beta Sf^{(r)} + \gamma y_\ep$ \quad\phantom{xxx!!!!} \texttt{\# Following \eqref{eq:contraction} and \eqref{eq:NLS}}\;
    	        $f^{(r+1)} \gets g / \varphi(g)$ \quad \phantom{xxxxxxxxxxxxxx!!!!} \texttt{\# Following \eqref{eq:varphi} and \eqref{eq:NLS}}\;
    	 }
    	 $\tilde{F}_{:,\ell} \gets f^{(r+1)}$
	}
    \lFor{$i = 1,\ldots,n$}{$\hat{Y}_i = \arg\max_{\ell} \tilde{F}_{i,\ell}$}
\end{algorithm}

The parameters $\alpha,\beta,\gamma$ allow us to tune the contributes of the
first-order local consistency, second-order local consistency, and global
consistency. For $\beta = 0$ we obtain a purely second-order method, which can
be useful when we do not have access to first-order data (e.g., we only have a
hypergraph). The case of $\alpha=0$ reduces to a normalized version of the standard LS as in \eqref{eq:LS}.
Furthermore, we can compute the iteration in \eqref{eq:NLS} efficiently.  Each
iteration requires one matrix-vector product and one ``tensor-martix'' product,
which only takes a single pass over the input data with computational cost
linear in the size of the data.
Finally, Algorithm~\ref{alg:N2LS} uses label smoothing in the initialization (the parameter $\ep$).
This will be useful for proving convergence results and can also improve generalization~\cite{muller2019does}.

\xhdr{The special case of a linear mixing function}
The mixing function $\sigma$ is responsible for the nonlinearity of the method.
The linear mixing function $\sigma(a,b) = a+b$ reduces \algname{} to a clique-expansion approach, which corresponds to a normalized version of~\cite{zhou2007hypergraph} for $\beta=0$.
To see this, let $K$ be the $n \times|\mathcal E|$ incidence matrix of the hypergraph $H$, where $K_{i,e} = 1$ if node $i$ is in the hyperedge $e$ and $K_{i,e}=0$ otherwise;
furthermore, let $W$ be the diagonal matrix of hyperedge weights $\tau(e)$, $e\in \mathcal E$.
Then $2 \big(KWK^\top\big)_{ij}  = \sum_k \A_{ijk} + \A_{ikj}$. 
Thus, for $\sigma(a,b) = a+b$,
\begin{align}
\S(f)_i =  \delta_i^{-1/2}\! \sum_{jk} \A_{ijk}\delta_j^{-1/2}\!f_j + \A_{ijk}\delta_k^{-1/2}f_k =  \delta_j^{-1/2} \!\sum_{j}(\sum_k \A_{ijk}+\A_{ikj})\delta_j^{-1/2}f_j = \big( \Theta f \big)_i,
\end{align}
where  $\Theta = 2 D_H^{-1/2} KWK^T D_H^{-1/2}$ is the normalized adjacency of the clique expansion graph~\cite{zhou2007hypergraph}.

\subsection{Global convergence and optimization framework}

Our \algname{} method extends standard LS in a natural way.  However, with the
nonlinear mixing function $\sigma$, it is unclear if the iterates even converge
or to what they might converge. In this section, we show that, remarkably,
\algname{} is globally convergent for a broad class of mixing functions 
and is minimizing a regularized objective similar to \eqref{eq:LS_loss}.
Lemmas used in the proofs are provided in Appendix~\ref{app:lemmas}.

\def\map{\mathit \Phi}
For convergence, we only require the mixing function to be positive and one-homogeneous. We recall below these two properties for a general map $\map$:
\begin{align}
  \text{\emph{Positivity}:}\; & \text{$\map(x)>0$ for all $x>0$}. \\
  \text{\emph{One-homogeneity}:}\; & \text{$\map(c \cdot x)= c \cdot \map(x)$ for all coefficients $c >0$ and all $x$}.
\end{align}
For technical reasons, we require entry-wise positive initialization. This is the reason for the smoothed membership vectors
$y_\ep = (1-\ep)Y_{:,\ell}+\ep\mathbf{1}$ in Algorithm~\ref{alg:N2LS}.
This assumption is not restrictive in practice as $\ep$ can be chosen fairly small, and we
can also interpret this as a type of label smoothing~\cite{muller2019does,szegedy2016rethinking}.

The following theorem says that the \algname{} iterates converge
for a broad class of mixing functions. This is a corollary of a more general result that we prove later (Theorem~\ref{thm:general-convergence}).
\begin{theorem}\label{thm:convergence}
  Let $f^{(r)}$ be the iterates in Algorithm~\ref{alg:N2LS}.
  If $\sigma$ is positive and one-homogeneous, then the sequence $\{f^{(r)}\}_{r}$ converges to a unique stationary point $f^*>0$ with $\varphi(f^*)=1$.
\end{theorem}

Even if the iterates converge, we would still like to know to what they converge.
We next show that for differentiable mixing functions $\sigma$, 
a scaled version of $f^*$ minimizes a regularized objective
that enforces local and global consistency.
For a smoothed membership vector $y_\ep = (1-\ep)Y_{:,\ell}+\ep \mathbf{1}$ with $0 < \ep < 1$, consider the loss function:
\begin{equation}\label{eq:loss}
    \vartheta(f)  = \frac 12 \bigg\{ \Big\| f - \frac{y_\ep}{ \varphi(y_\ep)} \Big\|_2^2 + \lambda \sum_{ij}A_{ij}\Big(\frac{f_i}{ \sqrt{d_i} }-\frac{f_j}{ \sqrt{d_j}}\Big)^2 + \mu   \sum_{ijk}\A_{ijk} \Big(\frac{f_i}{ \sqrt{\delta_i}}- \frac12 \sigma\Big(\frac{f_j}{ \sqrt{\delta_j}},\frac{f_k}{ \sqrt{\delta_k}}\Big) \Big)^2  \bigg\}
\end{equation}
As for the case of standard LS, $\vartheta$ consists of a global consistency
term regularized by a local consistency term. However, there are two main
differences. First, the global consistency term now considers a normalized
membership vector $\tilde y_\ep = y_\ep/\varphi(y_\ep)$.  As $\varphi(\tilde y_\ep)=1$,
we seek for a minimizer of $\vartheta$ in the slice $\{f:\varphi(f)=1\}$.
Second, the regularizer now combines the normalized
Laplacian term with a new tensor-based term that accounts for higher-order interactions.

Analogous to standard LS, \algname{} can be interpreted as a projected diffusion
process that spreads the input label assignment via the nonlinear gradient flow
corresponding to the energy function $\tilde \vartheta(f) = \vartheta(f) - \frac \mu 2 \varphi(f)^2$. In fact, for $y_\ep$ such that $\varphi(y_\ep)=1$, we have that
\begin{equation}\label{eq:EulerNLS}
    f-h\nabla\tilde \vartheta(f) = (1-h-\lambda h +\mu h)f  + h\lambda Sf + h \mu \S(f) + h y_\ep \, .
\end{equation}
(A proof of this identity is the proof of Theorem~\ref{thm:loss}.)
Thus, our \algname{} iterates in \eqref{eq:NLS} correspond to projected gradient descent applied to $\tilde\vartheta$, with step length $h$ chosen so that $(1-h)/h = \lambda + \mu$.
This is particularly useful in view of \eqref{eq:loss}, since $\tilde \vartheta$ and $\vartheta$ have the same minimizing points on $\{f:\varphi(f)=1\}$. Moreover, as  $\nabla \psi$ in \eqref{eq:ODE} can be interpreted as a discrete Laplacian operator on the graph~\cite{zhou2004learning},
we can interpret $\nabla \tilde \vartheta$, and thus $\S$,  as a hypergraph Laplacian operator,
which adds up to the recent literature on (nonlinear) Laplacians on hypergraphs~\cite{chan2018spectral,louis2015hypergraph}.

Unlike the standard label spreading, the loss functions $\vartheta$ and $\tilde \vartheta$ are not convex in general.
Thus, the long-term behavior $\lim_{t\to\infty}f(t)$ of the gradient flow $\dot f(t) = -\nabla \tilde \vartheta(f(t))$ is not straightforward.
Despite this, the next Theorem \ref{thm:loss} shows that the \algname{} can converge to a unique global minimizer of $\vartheta$ over the set of nonnegative vectors.
For this result, we need an additional assumption of differentiability on the mixing function $\sigma$. 
\begin{theorem}\label{thm:loss}
  Let $f^{(r)}$ be the sequence generated by Algorithm~\ref{alg:N2LS}.
  If $\sigma$ is positive, one-homogeneous, and differentiable, then the sequence $\{f^{(r)}\}_{r}$
  converges to the unique global solution of the constrained optimization problem
   \begin{equation}\label{eq:constrained_loss}
   \min \vartheta(f) \; \text{s.t. }f> 0 \text{ and } \varphi(f)=1.
  \end{equation}
\end{theorem}
\begin{proof}
Let 
$$
E_1(f) = \sum_{ij}A_{ij}(f_i/\sqrt{d_i}-f_j/\sqrt{d_j})^2, \qquad  E_2(f) = \sum_{ijk}\A_{ijk}(f_i/\sqrt{\delta_i}-\sigma(f_j/\sqrt{\delta_j},f_k/\sqrt{\delta_k})/2)^2
$$
and consider the following modified loss
\begin{equation*}
   \tilde  \vartheta(f)  = \frac 12 \Big\{ \Big\| f - \frac y{\varphi(y)} \Big\|^2 + \lambda E_1(f) + \mu E_2(f)  - \mu \varphi(f)^2 \Big\}, .
\end{equation*}
Clearly, when subject to $\varphi(f)=1$, the minimizing points of $\tilde\vartheta$ and those of $\vartheta$ in \eqref{eq:loss} coincide. We show that the gradient of the loss function $\tilde \vartheta$  vanishes on $f^*>0$ with $\varphi(f^*) = 1$ if and only if $f^*$ is a fixed point for the iterator of Algorithm \ref{alg:N2LS}. 

For simplicity, let us write $\tilde y = y/\varphi(y)$ with $y>0$.  We have $\nabla \|f-\tilde y\|^2 = 2(f-\tilde y)$ and  $\nabla E_1(f) = 2\L f=2(I-D_G^{-1}AD_G^{-1})f$. As for $E_2$, observe that from $\sum_{jk}\A_{ijk}=\delta_i$,we get 
\begin{align*}
    f^\top (D_H f-\A \sigma(f)) &= \sum_{i}\delta_i f_i^2 - \sum_{ijk}f_i \A_{ijk} \sigma(f_j,f_k) = \sum_{ijk} \A_{ijk} \big(f_i^2 - f_i\sigma(f_j,f_k)\big) \\
    &= \sum_{ijk} \A_{ijk} \Big(f_i - \frac{\sigma(f_j,f_k)}{2}\Big)^2 - \frac 1 4 \sum_{jk}B_{jk} \sigma(f_j,f_k)^2
\end{align*}
with $B_{jk} = \sum_i \A_{ijk}$. Thus, it holds that
$$
f^\top (D_H f-\A \sigma(f)) = E_2(D_H^{1/2}f) - \varphi(D_H^{1/2}f)^2\, .
$$
Now, since $\sigma(f)$ is one-homogeneous and differentiable, so is $F(f) = \P\sigma(f)$, and  using Lemma \ref{lem:Euler} we obtain $\nabla \{E_2(D_H^{1/2}f)-\varphi(D_H^{1/2}f)^2\} = 2(D_Hf - \A \sigma(f))$ which, with the change of variable $f\mapsto D_H^{-1/2}f$, yields
$$
\nabla \{E_2(f)-\varphi(f)^2\} = f-D_H^{-1/2}\A\sigma(D_H^{-1/2}f) 
$$

Altogether we have that
\begin{align*}
    \nabla \tilde \vartheta (f)  &= f-\tilde y + \lambda (I-D_G^{-1/2}AD_G^{-1/2})f + \mu \{ f-D_H^{-1/2}\A\sigma(D_H^{-1/2}f) \}\\
    &=(1+\lambda+\mu)f - \lambda D_G^{-1/2}AD_G^{-1/2}f - \mu D_H^{-1/2}\A\sigma(D_H^{-1/2}f) - \tilde y,
\end{align*}
which implies  that $f^*\in \RR^n_{++}/\varphi$ is such that  $\nabla \tilde \vartheta (f^*) = 0$ if and only if $f^*$ is a fixed point of \algname{}, i.e.\ 
$$
 f^* = \alpha D_H^{-1/2}\A\sigma(D_H^{-1/2}f^*) + \beta  D_G^{-1/2}AD_G^{-1/2}f^* +\gamma  \tilde y
$$
with $\lambda = \beta/\gamma$, $\mu=\alpha/\gamma$ and $\alpha+\beta+\gamma=1$.

Finally, by Theorem \ref{thm:convergence} we know that the \algname{} iterations in Algorithm \ref{alg:N2LS} converge to $f^*>0$, $\varphi(f^*)=1$ for all positive starting points $f^{(0)}$. Moreover, $f^*$ is the unique fixed point in the slice $\RR^n_{++}/\varphi = \{f>0:\varphi(f) =1\}$. As $\vartheta$ and $\tilde \vartheta$ have the same minimizing points on that slice, this shows that $f^*$ is the global solution of $\min \{ \vartheta(f) : f\in \RR^n_{++}/\varphi\}$, concluding the proof.
\end{proof}

\subsection{Main convergence result and extension to other orders} \label{sec:fo_nonlin}
Theorem~\ref{thm:convergence} is a direct consequence of our main convergence theorem below,
where $F(f) = \alpha D_H^{-1/2}\A\sigma(D_H^{-1/2}f)+\beta D_G^{-1/2}AD_G^{-1/2}f$ and $y = \gamma y_\ep$.
By Lemma~\ref{lem:upper-bound} all of the assumptions for Theorem~\ref{thm:general-convergence} are satisfied, and the convergence follows.

\begin{theorem}\label{thm:general-convergence}
    Let $F:\RR^n\to\RR^n$ and $\varphi:\RR^n\to\RR$ be positive and one-homogeneous mappings, and let $y$ be a positive vector. If there exists $C>0$ such that $F(f)\leq Cy$,  for all $f$ with $\varphi(f)=1$, then, for any $f^{(0)}>0$, the sequence 
  $$
  g^{(r)} = F(f^{(r)}) + y \qquad f^{(r+1)} = g^{(r)}/\varphi(g^{(r)})
  $$
  converges to $f^*>0$, unique fixed point of $F(f)+y$ such that $\varphi(f^*)=1$. 
\end{theorem}

This general convergence result makes it clear how to transfer the second-order setting discussed in this work to hyperedges of any order.
For example, nonlinearities could be added to the graph term as well. A nonlinear purely first-order label spreading has the general form
$f^{(r+1)} = \beta M \sigma(f^{(r)}) + \gamma y$, 
where $M \sigma(f)$ is the vector with entries $M\sigma(f)_i = \sum_j M_{ij}\sigma(f_j)$, $\sigma$ is a nonlinear map and $M$ is a graph matrix. This
formulation is quite general and provides new convergence results for
nonlinear graph diffusions recently proposed by Ibrahim and Gleich~\cite{ibrahim2019nonlinear}. 
For simplicity and in the interest of space, we do not use this extension in our
experiments, but out of theoretical interest we note that Theorem~\ref{thm:general-convergence} would allow us to take this additional nonlinearity into account.  

\begin{proof}
  Let $\varphi:\RR^n\to\RR$ and $F:\RR^n\to\RR^n$ be one-homogeneous and positive maps. Let $\RR^n_+$ and $\RR^n_{++}$ denote the set of entrywise nonnegative and entrywise positive vectors in $\RR^n$, respectively. Also, let $\RR^n_{++}/\varphi = \{f \in \RR^n_{++} : \varphi(f) = 1\}$.
This result falls within the family of Denjoy-Wolff type theorems for  nonlinear mappings on abstract cones, see e.g., Lemmens et
al.~\cite{lemmens2018denjoy}. Here we provide a simple and self-contained proof. Define
\[
G(f) = F(f) + y \qquad \text{and}\qquad \tilde G(f) = \frac{G(f)}{\varphi(G(f))}\, .
\]
Notice that $f^{(r+1)}=\tilde G(f^{(r)})$, $r=0,1,2,\dots$ and that, by assumption, $G(f)>0$ for all $f>0$. Thus $\varphi(G(f))>0$ and $\tilde G(f)$ is well defined on $\RR^n_{++}$. Moreover, $\tilde G$ preserves the cone slice $\RR^n_{++}/\varphi$ i.e.\ $\tilde G(f)\in \RR^n_{++}/\varphi$ for all $f\in \RR^n_{++}/\varphi$. For two points $u,v \in \RR^n_{++}$ let 
$$
d(u,v) = \log\Big(\frac{M(u/v)}{m(u/v)}\Big)
$$
be the Hilbert distance, 
where $M(u/v) = \max_{i}u_i/v_i$ and $m(u/v) = \min_i u_i/v_i$. 
As  $\sigma$ is one-homogeneous also $\varphi$ is one-homogeneous and thus $\RR^n_{++}/\varphi$ equipped with the Hilbert distance is a complete metric space,  \cite{gautier2019perron} e.g. In order to conclude the proof it is sufficient to show that  $\tilde G$ is a contraction with respect to such metric. In fact, the sequence $f^{(r)}$ belongs to $\RR^n_{++}/\varphi$ for any $f^{(0)}$ and since $(\RR^n_{++}/\varphi,d)$ is complete, the sequence $f^{(r)}$ converges to the unique fixed point $f^*$ of $\tilde G$ in $\RR^n_{++}/\varphi$.

We show below that  $\tilde G$ is a contraction. To this end, first note that by definition we have 
$
d(\tilde G(u),\tilde G(v)) = d(G(u),G(v))
$. 
Now note that, as $\sigma$ is homogeneous and order preserving, for any $u,v\in \RR^n_{++}$ we have $m(u/v)F(v)=F(m(u/v)v)\leq F(u)$, $M(u/v)F(v)=F(M(u/v)v)\geq F(u)$, $m(u/v)\varphi(v)=\varphi(m(u/v)v)\leq \varphi(u)$ and $M(u/v)\varphi(v)=\varphi(M(u/v)v)\geq \varphi(u)$. Moreover, for any $u,v\in \RR^n_{++}/\varphi$ it holds
$$
m(u/v) = \varphi(u)m(u/v)\leq \varphi(v) = 1 \leq \varphi(u)M(u/v) = M(u/v),
$$
thus $m(u/v)\leq 1 \leq M(u/v)$. 
By assumption there exists $C>0$ such that $F(u)\leq Cy$, for all $u\in \RR^n_{++}/\varphi$. Therefore, for any $u,v\in \RR^n_{++}/\varphi$, we have
\begin{align*}
(m(u/v)C &+ 1)(F(v)+ y)= (m(u/v)C+ m(u/v)- m(u/v)+1)F(v)+(m(u/v)C + 1) y \\
&= m(u/v)CF(v) +  m(u/v)F(v) + (1-m(u/v))F(v) + (m(u/v)C + 1) y\\
&\leq CF(u) +  F(u) + (1-m(u/v))Cy + (m(u/v)c + 1) y \\
&= (C+1)(F(u)+ y)
\end{align*}
where we used the fact that $(1-m(u/v))\geq 0$ to get $(1-m(u/v))F(v)\leq (1-m(u/v))Cy$. Thus
$$
m(G(u)/G(v)) = m\big((F(u)+ y) / (F(v)+ y)\big) \geq (m(u/v)C+1)/(C+1)\, .
$$
Similarly, as $(1-M(u/v))\leq 0$,  we have 
\begin{align*}
(M(u/v)C + 1)(F(v)+ y)&\geq C F(u) + F(u) + (1-M(u/v))Cy + (M(u/v)C + 1) y \\
&= (C+1)(F(u)+ y)
\end{align*}
which gives $M(G(u)/G(v)) \leq (M(u/v)C+1)/(C+1)$. Therefore, 
$$
d(G(u),G(v)) = \log\Big(\frac{M(G(u)/G(v))}{m(G(u)/G(v))} \Big) \leq \log\Big( \frac{M(u/v)C+1}{m(u/v)C+1}\Big) = \log\Big( \frac{M(u/v)+\delta}{m(u/v)+\delta} \Big)
$$
with $\delta = 1/C$. Finally, using Lemma \ref{lem:logarithm} we get
$$
d(\tilde G(u),\tilde G(v)) = d(G(u),G(v)) \leq \left(\frac{M(u/v)}{M(u/v)+\delta}\right) d(u,v) < d(u,v)
$$
which shows that $\tilde G$ is a contraction in the complete metric space $(\RR^n_{++}/\varphi,d)$.
\end{proof}

\section{Experiments}
We now perform experiments on both synthetic and real-world data.
Our aim is to compare the performance of standard (first-order) label spreading algorithm with our second-order methods, using different mixing functions $\sigma$. 
Based on \eqref{eq:loss}, a natural class of functions are the one-parameter
family of generalized $p$-means scaled by factor $2$, i.e.,
$\sigma_p(a, b) = 2 \left((a^p+b^p) / 2\right)^{1/p}$.
In particular, we consider
the following generalized means:
\renewcommand{\arraystretch}{1}
\begin{center}
{
    \begin{tabular}{c @{\qquad\qquad} ccccc}
    \toprule
     name & arithmetic & harmonic & $L^2$ & geometric & maximum \\
     \midrule
     $p$  & $p=1$ & $p=-1$ & $p=2$ & $p\to 0$ & $p \to \infty$\\
      \midrule
     $\sigma_p(a,b)$ & $(a+b)$ & $4\big(\frac 1a +\frac 1b \big)^{-1}$ & $2\sqrt{\frac{a^2+b^2}{2}}$ & $2\sqrt{ab}$ &
     $2 \cdot \max(a, b)$ \\
     \bottomrule
    \end{tabular}
}
\end{center}

The function $\sigma_p$ is one-homogeneous, positive, and order-preserving for all values of $p\in\RR$, including the limit cases of $p\in\{0,+\infty\}$.
Thus, by Theorem~\ref{thm:convergence}, Algorithm~\ref{alg:N2LS} converges for any choice of $p$.
The maximum function, however, is not differentiable and thus we cannot prove the computed limit of the sequence $f^{(r)}$ optimizes the loss function \eqref{eq:loss}
(but one can approximate the maximum with a large finite $p$, and  use Theorem~\ref{thm:loss} on such smoothed max).
Also, as shown above, the case $p = 1$ (i.e., $\sigma_p$  linear) essentially corresponds to a clique expansion method,
so Algorithm~\ref{alg:N2LS} is nearly the same as the local and global
consistency method for semi-supervised learning~\cite{zhou2004learning},
where the adjacency matrix is a convex combination of the clique expansion
graph induced by the tensor $\A$ and the graph $A$.
In our experiments, we use at most 40 iterations within Algorithm~\ref{alg:N2LS},
a stopping tolerance of 1e-5, and smoothing parameter $\varepsilon$ = 0.01. Other tolerances and other $\ep$ gave similar~results. The code for implementing \algname{} is available at \url{https://github.com/doublelucker/nhols}.

In addition to comparing against standard label spreading, we also compare
against two other techniques. The first is hypergraph total variation (HTV)
minimization, which is designed for clustering hypergraphs with larger
hyperedges but is still applicable to our setup~\cite{hein2013total}. This is a
state-of-the-art method for pure hypergraph data.  The second is a graph neural
network (GNN) approach, which is broadly considered state-of-the-art for
first-order methods. More specifically, we
use GraphSAGE~\cite{hamilton2017inductive} with two layers, ReLU activation, and
mean aggregation, using node features if available and one-hot encoding features
otherwise. Neither of these baselines incorporate both first- and second-order
information in the graph, as \algname{}.

All of the algorithms that we use have hyperparameters.
For the label spreading and HTV methods, we run 5-fold cross validation
with label-balanced 50/50 splits over a small grid to choose these hyperparameters.
For standard label spreading, the grid is $\beta \in \{0.1, 0.2, \ldots, 0.9\}$.
For higher-order label spreadings, the grid is $\alpha \in \{0.3,0.4,\ldots,0.8\}$ and
$\beta \in \{0.1, 0.25, 0.40, 0.55\}$ (subject to the constraint that $\alpha + \beta < 1$).
HTV has a regularization term $\lambda$, which we search over $\lambda =(1-\beta)/\beta$ for $\beta \in  \{0.1, 0.2, \ldots, 0.9\}$ (i.e., the same grid as LS).
We choose the parameters that give the best average accuracy over the five folds.
The GNN is much slower, so we split the labeled data into a training and
validation sets with a label-balanced 50/50 split (i.e., a 1-fold cross validation,
as is standard for training such models).
We use  ADAM optimizer with default $\beta$ parameters and search
over learning rates $\eta \in \{0.01, 0.001, 0.0001\}$ and
weight decays $\omega \in \{0, 0.0001\}$,  using 15 epochs. 

\subsection{Synthetic benchmark data}
\def\pin{ p_{\text{in}} }
\def\pout{ p_{\text{out}} }

\begin{figure}[t]
    \centering
    \includegraphics[width=\textwidth,clip,trim=.5cm 0 0 0]{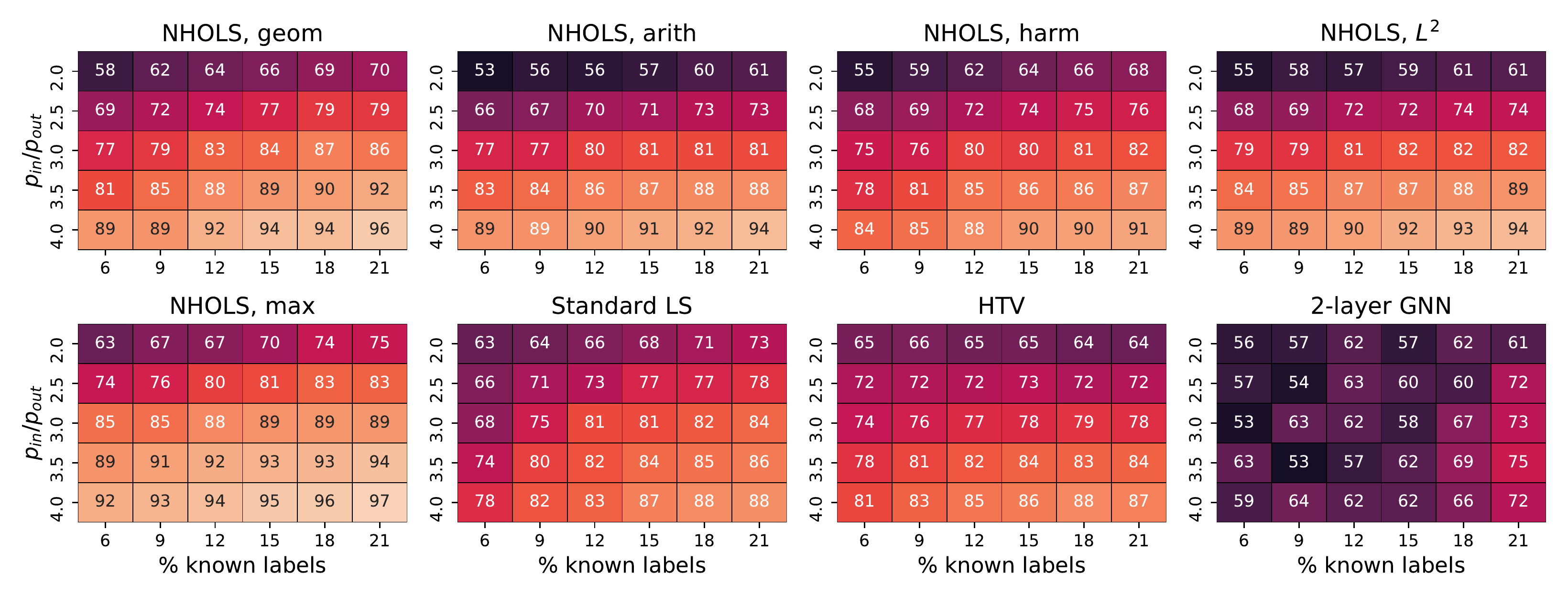}
    \caption{Accuracy on synthetic stochastic block models.
      Each table corresponds to a different method, and table entries are the average accuracy over 10 random instances with the given parameter settins.
    Overall, the various \algname{} methods perform much better than the baselines.}
    \label{fig:sbm}
\end{figure}

\begin{figure}
  \centering
  \includegraphics[width=.5\columnwidth]{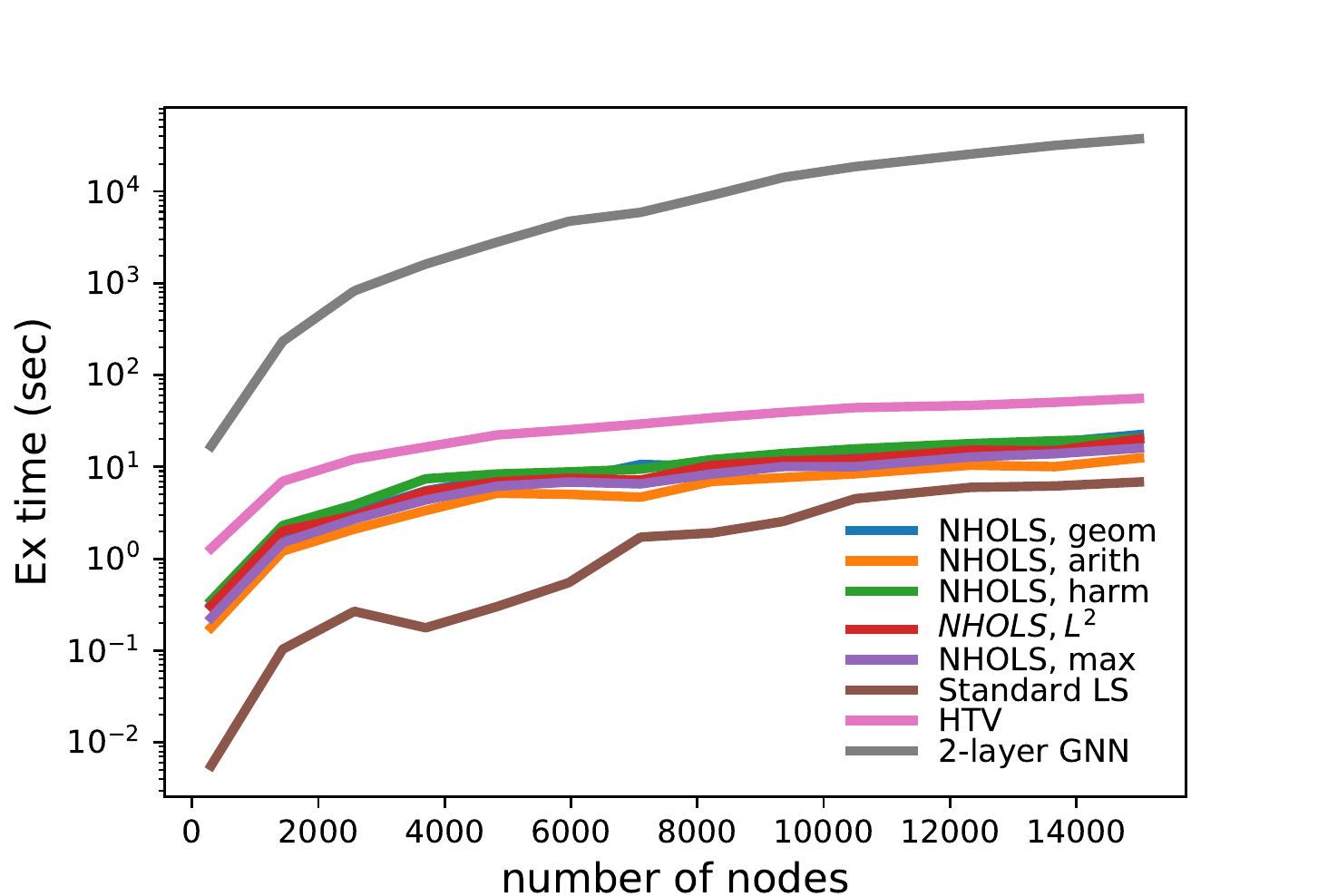}
  \caption{Running times of various algorithms for a single hyperparameter setting with SBM data.
  \algname{} scales linearly with the number of nonzeros in $\S$ and $S$
  and costs a little more than standard LS. HTV is a bit little slower, and the GNN
  is several orders of magnitude slower for even modest network sizes.}
  \label{fig:running_times}
\end{figure}
s
We first compare the semi-supervised learning algorithms on synthetic graph data generated with the Stochastic Block Model (SBM).
The SBM is a generative model for graph data with prescribed cluster structure.
We use the variant of an SBM with two parameters --- $\pin$ and $\pout$ --- which designate the edge probabilities within the same label and across different labels, respectively.
Generative block models are a common benchmark to test the performance of semi-supervised learning methods~\cite{kanade2016global,kloumann2017block,mercado2019generalized,mossel2016local}.
Here we analyze the performance of different methods on random graphs drawn from the SBM where nodes belong to three different classes of size (number of nodes) 100, 200 and 400.
We sample graphs for different values of the parameters $\pin$ and $\pout$; precisely, we fix $\pin=0.1$ and let $\pout = \pin/\rho$ for $\rho \in  \{2, 2.5, 3, 3.5, 4\}$.
We test the various algorithms for different percentages of known labels per each class ($\{6\%,9\%,12\%,15\%,18\%,21\%\}$).
With this setting, small values of $\rho$ correspond to more difficult classification problems.

 The colored tables in Figure~\ref{fig:sbm} show the average clustering accuracy over ten random samples for each SBM setting and each percentage of input labels.
We observe that the nonlinear label spreading methods perform best overall, with the maximum function performing the best in nearly all the cases.
The performance gaps can be quite substantial. For example, when only 6\% of the labels are given,
\algname{} achieves up to $92\%$  mean accuracy, while the baselines do not achieve greater than $81\%$.

Moreover, \algname{} scales linearly with the number of nonzero elements of $\S$
and $S$ and thus is typically just slightly more expensive than standard
LS. This is illustrated in Figure~\ref{fig:running_times}, where we compare mean
execution time over ten runs for different methods. We generate random SBM
graphs with three labels of equal size, increasing number of nodes $n$ and edge
probabilities $\pin = \log(n)^2/n$ and $\pout=\pin/3$. Each method is run with a
fixed value of the corresponding hyperparameters, chosen at random each time. We
note that HTV and GNN are around one and three-to-five orders of magnitude
slower than \algname{},~respectively.

\subsection{Real-world data}

\renewcommand{\arraystretch}{1}
\begin{table}[t]
\caption{Mean accuracy over five random samples of labeled nodes
over six datasets and four percentages of labeled nodes.
We compare our \algname{} method using five different mixing functions
to standard LS, Hypergraph Total Variation
minimization~\cite{hein2013total} and a Graph Neural Network model~\cite{hamilton2017inductive}.
Incorporating higher-order information
into label spreading always improves performance.
HTV is sometimes competitive, 
and the GNN has poor performance.
}
\label{tab:realworld}
\centering
{\footnotesize
\begin{tabular}{l  r  cccc   cccc}
\toprule
& & \multicolumn{4}{c}{Rice31 (n = 3560)} & \multicolumn{4}{c}{Caltech36 (n = 590)} \\
\cmidrule(lr){3-6} \cmidrule(lr){7-10}
method & \% labeled & 5.0\% & 10.0\% & 15.0\% & 20.0\% & 5.0\% & 10.0\% & 15.0\% & 20.0\% \\
\midrule
NHOLS, arith & & 88.0 & \textbf{89.7} & \textbf{90.8} & 91.1 & 80.9 & 82.0 & 85.3 & 85.4 \\
NHOLS, harm & & \textbf{88.1} & 89.6 & 90.7 & \textbf{91.3} & 79.6 & 81.7 & 84.6 & 84.9 \\
NHOLS, $L^2$ & & 88.0 & \textbf{89.7} & 90.7 & 91.2 & 80.7 & \textbf{82.3} & 85.0 & \textbf{85.3} \\
NHOLS, geom & & 87.9 & \textbf{89.7} & 90.7 & 91.2 & 80.3 & \textbf{82.3} & 84.8 & 85.0 \\
NHOLS, max & & 87.6 & 89.3 & 90.3 & 90.7 & \textbf{81.0} & \textbf{82.3} & \textbf{86.0} & \textbf{85.3} \\
Standard LS & & 83.1 & 87.6 & 89.5 & 90.6 & 70.1 & 76.6 & 80.9 & 81.9 \\
HTV & & 81.6 & 85.7 & 87.7 & 90.0 & 66.1 & 76.1 & 75.9 & 81.8 \\
2-layer GNN & & 54.2 & 69.2 & 79.1 & 80.6 & 44.6 & 44.3 & 61.6 & 54.2 \\
\midrule
& & \multicolumn{4}{c}{optdigits (n = 5620)} & \multicolumn{4}{c}{pendigits (n = 10992)} \\
\cmidrule(lr){3-6} \cmidrule(lr){7-10}
method & \% labeled & 0.4\% & 0.7\% & 1.0\% & 1.3\% & 0.4\% & 0.7\% & 1.0\% & 1.3\% \\
\midrule
NHOLS, arith & & 94.9 & 95.7 & 96.2 & \textbf{97.5} & 91.1 & 91.4 & 95.4 & 95.8 \\
NHOLS, harm & & 93.2 & 95.5 & 95.8 & 97.2 & 90.7 & 91.2 & 95.4 & 95.6 \\
NHOLS, $L^2$ & & \textbf{95.5} & 95.7 & \textbf{96.3} & 97.7 & 91.1 & 91.5 & 95.4 & 95.7 \\
NHOLS, geom & & 94.0 & 95.6 & 95.9 & 97.4 & 91.0 & 91.3 & 95.4 & 95.7 \\
NHOLS, max & & 94.8 & \textbf{95.8} & 95.9 & 97.4 & \textbf{93.6} & \textbf{92.8} & \textbf{95.9} & \textbf{95.9} \\
Standard LS & & 93.8 & 93.8 & 95.6 & 96.7 & 91.3 & 92.0 & 95.3 & 95.0 \\
HTV & & 87.0 & 90.9 & 93.3 & 94.1 & 82.2 & 91.3 & 93.3 & 93.3 \\
2-layer GNN & & 52.1 & 62.6 & 70.9 & 73.6 & 56.8 & 67.7 & 63.0 & 64.4 \\
\midrule
& & \multicolumn{4}{c}{MNIST (n = 60000)} & \multicolumn{4}{c}{Fashion-MNIST (n = 60000)} \\
\cmidrule(lr){3-6} \cmidrule(lr){7-10}
method & \% labeled & 0.1\% & 0.3\% & 0.5\% & 0.7\% & 0.1\% & 0.3\% & 0.5\% & 0.7\% \\
\midrule
NHOLS, arith & & 92.6 & 95.3 & 95.7 & \textbf{95.9} & 71.0 & \textbf{75.6} & \textbf{77.9} & 78.2 \\
NHOLS, harm & & 91.8 & 94.9 & 95.6 & 95.8 & 70.8 & 75.3 & \textbf{77.9} & 78.1 \\
NHOLS, $L^2$ & & \textbf{92.7} & \textbf{95.4} & \textbf{95.8} & \textbf{95.9} & 
\textbf{71.3} & \textbf{75.6} & 77.7 & \textbf{78.3} \\
NHOLS, geom & & 92.0 & 95.0 & 95.6 & 95.8 & 70.6 & 75.3 & 77.8 & 77.9 \\
NHOLS, max & & 92.3 & 95.3 & 95.7 & 95.8 & 69.7 & 75.4 & 77.0 & 77.9 \\
Standard LS & & 87.6 & 92.2 & 93.6 & 94.1 & 68.6 & 73.4 & 76.5 & 77.4 \\
HTV & & 79.7 & 88.3 & 90.1 & 90.7 & 59.8 & 68.6 & 70.2 & 72.0 \\
2-layer GNN & & 60.5 & 77.5 & 81.5 & 84.9 & 63.2 & 71.8 & 73.0 & 73.7 \\
\bottomrule
\end{tabular}
}
\end{table}

We also analyze the performance of our methods on six real-world datasets
(Table~\ref{tab:realworld}). The first two datasets come from relational
networks, namely Facebook friendship graphs of Rice University and Caltech from the Facebook100
collection~\cite{traud2012social}. The labels are the dorms in which the students
live --- there are 9 dorms for Rice and 8 for Caltech.
Facebook friendships form edges, and the tensor entries correspond
to triangles (3-cliques) in the graph. 
We preprocessed the graph data by first removing students with missing labels
and then taking the largest connected component of the resulting graph.
The GNN uses a one-hot encoding for features.

The next four graphs are derived from point clouds:
optdigits~\cite{UCI,xu1992methods},
pendigits~\cite{alimoglu1996methods,UCI}, 
MNIST~\cite{lecun2010mnist}, and
Fashion-MNIST~\cite{xiao2017fashion}. 
Each of these datasets has 10 classes, corresponding
to one of 10 digits or to one of 10 fashion items.
In these cases, we first create 7-nearest-neighbor graphs.
Tensor entries correspond to triangles in this graph.
We give the GNN an additional advantage
by providing node features derived from the data points.
The optdigits and pendigits datasets come with several hand-crafted features;
and we use the embedding in first 10 principal components for node features.
The MNIST and Fashion-MNIST datasets contain the raw images;
here, we use an embedding from the first 20 principal components of the images as node features for the GNN.
(We also tried one-hot encodings and the raw data points as features,
both of which had much worse performance.)

We recorded the mean accuracy of five random samples of labeled nodes (Table~\ref{tab:realworld}).
We find that incorporating higher-order information with some mixing function 
improves performance of LS in nearly all cases and also outperforms the baselines.
The absolute accuracy of these methods is also quite remarkable; for example,
the $L^2$ mixing function achieves 92.7\% mean accuracy with just 0.1\% of MNIST
points labeled (6 samples per class) and the maximum mixing function achieves
93.6\% mean accuracy with just 0.4\% of pendigits points labeled (5 samples per class).
Also, a single mixing function tends to have the best performance on a given dataset,
regardless of the number of available labels (e.g., the maximum 
mixing function for Caltech36 and pendigits, the $L^2$ mixing function
for MNIST and Fashion-MNIST).

Finally, the GNN performance is often poor, even in cases where meaningful node features are available.
This is likely a result of having only a small percentage of labeled examples.
For example, it is common to have over 15\% of the nodes labeled just as a validation set
for hyperparameter tuning~\cite{kipf2017semi}.
Still, even with 20\% of nodes labeled and hyperparameter tuning, the GNN
performs much worse than all other methods on the Facebook graphs.

\section{Discussion}

We have developed a natural and substantial extension of traditional label
spreading for semi-supervised learning that can incorporate a broad range of
nonlinearities into the spreading process. These nonlinearities come from mixing
functions that operate on higher-order information associated with the data.
Given the challenges in developing optimization results involving nonlinear
functions, it is remarkable that we can achieve a sound theoretical framework
for our expressive spreading process. We provided guarantees on convergence of
the iterations to a unique solution, and we showed that the process optimizes a
meaningful objective function. 
Finally, the convergence result in Theorem~\ref{thm:general-convergence}
is more general than what we considered in our experiments. This provides a
new way to analyze nonlinearities in graph-based methods that we expect will be
useful for future research.

\subsection*{Acknowledgments} We thank Matthias Hein for supplying the code that implements the HTV algorithm in \cite{hein2013total}.
We thank David Gleich and Antoine Gautier for helpful discussions. 
ARB is supported by NSF Award DMS-1830274,
ARO Award W911NF19-1-0057,
ARO MURI,
and JPMorgan Chase \& Co.
FT is partially supported by INdAM-GNCS.

This research was funded in part by JPMorgan Chase \& Co. Any views or opinions expressed herein are
solely those of the authors listed, and may differ from the views and opinions expressed by JPMorgan Chase
\& Co. or its affiliates. This material is not a product of the Research Department of J.P. Morgan Securities
LLC. This material should not be construed as an individual recommendation for any particular client and is
not intended as a recommendation of particular securities, financial instruments or strategies for a particular
client. This material does not constitute a solicitation or offer in any jurisdiction.


\appendix
\section{Lemmas used in proofs}\label{app:lemmas}

\begin{lemma}\label{lem:logarithm}
Let $a,b,c>0$. Then
$$
\log\Big(\frac{a+c}{b+c}\Big)\leq \frac{a}{a+c}\log\Big(\frac{a}{b}\Big)
$$
\end{lemma}
\begin{proof}
Let $g(x) = x\log(x)$. We have 
$$
\frac{a+c}{b+c}\, \log\Big(\frac{a+c}{b+c}\Big) = g\Big(\frac{a+c}{b+c}\Big) = g\Big(\frac{b}{b+c}\, \frac a b + \frac{c}{b+c}\Big) \, .
$$ 
As $\frac{b}{b+c}+\frac{c}{b+c}=1$ and $g$ is convex, we can apply Jensen's inequality to get
$$
g\Big(\frac{b}{b+c}\, \frac a b + \frac{c}{b+c}\Big) \leq \frac{b}{b+c}\, g\Big(\frac a b\Big) + \frac{c}{b+c}\, g(1) = \frac{a}{b+c} \log\Big(\frac a b\Big)\, .
$$
Combining all together we get 
$$
\frac{a+c}{b+c}\, \log\Big(\frac{a+c}{b+c}\Big)  \leq \frac{a}{b+c} \log\Big(\frac a b\Big)
$$
which yields the claims.
\end{proof}

\begin{lemma}\label{lem:upper-bound}
Assume that $\sigma$ is one-homogeneous and positive and that both $y\in \RR^n_{++}$ and $\alpha D_H^{-1/2}\A\sigma(D_H^{-1/2}f)+\beta D_G^{-1/2}AD_G^{-1/2}f \in \RR^n_{++}$, for every  $f\in \RR^n_{++}$. Then 
$$
\alpha D_H^{-1/2}\A\sigma(D_H^{-1/2}f)+\beta D_G^{-1/2}AD_G^{-1/2}f  \leq Cy
$$
for all $f\in \RR^n_{++}/\varphi$.
\end{lemma}
\begin{proof}
Since we are assuming that every node has hyper-degree $\delta_i = \sum_{jk}\A_{ijk}>0$, for every $i$ there exist $j$ and $k$ such that $ijk$ is a hyperedge. Thus, if $U\subseteq V\times V$ is the set of nonzero entries of $B$, then $V\subseteq U$. Thus if $f>0$ and $\varphi(f)=1$ then $f$ must be entrywise bounded. Hence, as $\varphi$ is positive and one-homogeneous,  we can choose
$$
C = \frac{1}{\min_i y_i} \max_{i=1,\dots, n}\max_{f\in \RR^n_{++}} \frac{(\alpha D_H^{-1/2}\A\sigma(D_H^{-1/2}f)+\beta D_G^{-1/2}AD_G^{-1/2}f)_i}{\varphi(f)} < +\infty
$$ 
to obtain the claim.
\end{proof}

\begin{lemma}\label{lem:Euler}
Let $E:\RR^n\to\RR_+$ be defined by 
\begin{equation}\label{eq:E_hom}
E(f) = \frac{f^\top(D_Hf - F(f) )}{2}
\end{equation}
with $F:\RR^n\to\RR^n$ differentiable and such that $F(f)\geq 0$ for all $f\geq 0$ and $F(\alpha f) = \alpha F(f)$ for all $\alpha\geq 0$. Then $\nabla E(f) =  D_H f - F(f)$.
\end{lemma}
\begin{proof}
This is a relatively direct consequence of Euler's theorem for homogeneous functions. For completeness, we provide a self-contained proof here. Consider the function $G(\alpha) = F(\alpha f) - \alpha F(f)$. Then $G$ is differentiable and $G(\alpha)=0$ for all $\alpha \geq 0$. Thus, $G'(\alpha)= 0$ for all $\alpha$ in a neighborhood of $\alpha_0=1$. For any such $\alpha$ we have
$$
G'(\alpha) = \alpha JF(\alpha f)f -  F(f) =0
$$
where $JF(f)$ denotes the Jacobian of $F$ evaluated at $f$. 
Evaluating $G'$ on $\alpha =1$ we get $JF(f)f = F(f)$. Therefore
\begin{align*}
2\nabla E(f) &=  \nabla \{f^\top(D_H f- F(f))\} 
= D_H f-   F(f) + \Big(D_H-JF(f)\Big)f 
= 2D_H f - 2F(f)
\end{align*}
which gives us the claim.
\end{proof}

\end{document}